\documentclass{article}
\usepackage{amsfonts,amsmath,amsthm,amssymb}

\usepackage{spconf,amsmath,graphicx,hyperref}
\usepackage{multirow}
\usepackage{float}
\usepackage{color}
\usepackage{arydshln}

\newtheorem{theorem}{Theorem}

\newtheorem{proposition}{Proposition}
\title{EMPEROR: Efficient Moment-Preserving Representation of Distributions}
\name{Xinran Liu$^{\star}$ \qquad Shansita D. Sharma$^{\star}$ \qquad Soheil Kolouri$^{\star \dagger}$}
\address{$^{\star}$ Department of Computer Science, Vanderbilt University\\
$^{\dagger}$Department of Electrical and Computer Engineering, Vanderbilt University}
\begin{document}
%
\maketitle
\begin{abstract}
We introduce EMPEROR (Efficient Moment-Preserving Representation of Distributions), a mathematically rigorous and computationally efficient framework for representing high-dimensional probability measures arising in neural network representations. Unlike heuristic global pooling operations, EMPEROR encodes a feature distribution through its statistical moments. Our approach leverages the theory of sliced moments: features are projected onto multiple directions, lightweight univariate Gaussian mixture models (GMMs) are fit to each projection, and the resulting slice parameters are aggregated into a compact descriptor. We establish determinacy guarantees via Carleman’s condition and the Cramér--Wold theorem, ensuring that the GMM is uniquely determined by its sliced moments, and we derive finite-sample error bounds that scale optimally with the number of slices and samples. Empirically, EMPEROR captures richer distributional information than common pooling schemes across various data modalities, while remaining computationally efficient and broadly applicable.
\end{abstract}

\section{Introduction}
\label{sec:intro}

Modern AI systems routinely compress rich, high-dimensional sets of features/tokens into a single vector via permutation-invariant pooling or a special aggregation token. Popular choices such as global average pooling \cite{lin2013network} and CLS-style attention pooling \cite{dosovitskiyimage} are computationally attractive but collapse the underlying distribution of features without guarantees on what information is preserved. This heuristic reduction can hinder interpretability, robustness, and data efficiency, and has motivated alternatives that try to encode more distributional structure \cite{Zhang2020FSPool,mialon2021a,kolouri2021wasserstein,naderializadeh2021PSWE,kothapalli2024equivariant}. However, most existing approaches emphasize empirical performance over principled recoverability or quantifiable fidelity to the original feature distribution.

In this paper, we propose EMPEROR, an Efficient Moment-Preserving Representation of Distributions, that treats a layer’s features as samples from a finite positive measure and encodes that measure through its moments. The core idea is to replace ambiguous, high-dimensional moment estimation with \emph{sliced moments}: we project features onto multiple directions, fit lightweight \emph{univariate} Gaussian mixture models (GMMs) to each projection, and aggregate the resulting slice parameters into a compact descriptor. Theoretically, sliced moments determine the multivariate measure under mild conditions (via Carleman + Cramér–Wold), and specializing to GMMs yields explicit, stable moment formulas. Practically, univariate fits avoid the $O(d^2)$ burden of full covariances, are robust and scalable, and give closed-form moments that can be assembled degree-by-degree. We further analyze the conditioning of the slice design, showing that the reconstruction error of degree-$k$ moments decays as $L^{-1/2}$ with the number of slices (and $N^{-1/2}$ with samples), enabling a tunable accuracy–cost trade-off. EMPEROR thus provides a mathematically principled alternative to heuristic pooling.

Our contributions are fourfold:  
(i) \textbf{Theory.} We establish a sliced-moment determinacy result for finite measures and instantiate it for multivariate GMMs, ensuring identifiability from one-dimensional projections.  
(ii) \textbf{Algorithm.} We introduce a simple, parallelizable pipeline that fits $K$-component \emph{univariate} GMMs across $L$ slices and produces a fixed-size, moment-preserving descriptor without cross-slice coupling or $O(d^2)$ parameter growth.  
(iii) \textbf{Statistics.} We provide finite-sample error bounds for recovering degree-$k$ multivariate moments from noisy sliced estimates, with explicit $L^{-1/2}$ and $N^{-1/2}$ rates governed by the smallest eigenvalue of a slice design matrix.  
(iv) \textbf{Practice.} We demonstrate that EMPEROR captures distributional information more faithfully than common pooling schemes across diverse data modalities, while maintaining competitive efficiency.

\vspace{-.1in}
\section{Method}
\label{sec:method}
\vspace{-.1in}

Let $\rho \in \mathcal{M}_+(\mathbb{R}^d)$ denote a \emph{finite} positive Borel measure on $\mathbb{R}^d$.
In this work we primarily deal with empirical measures arising in neural network representations, but we keep the definitions general so the framework applies to arbitrary \emph{finite} positive measures. Our goal is to construct fixed-dimensional vector representations of $\rho$ that preserve its statistical moments. We begin with the necessary definitions, then address two fundamental questions: (i) which classes of distributions can be uniquely determined by their moments, and (ii) what is the minimal set of parameters required to represent these moments without redundancy?

\vspace{-.1in}
\subsection{The moment problem}
\label{append:mp}
\vspace{-.1in}

The classical moment problem \cite{stieltjes1894recherches,shohat1950problem,aheizer1965classical} asks whether a sequence of real numbers $(m_k)_{k=0}^\infty$ can be realized as the moments of a finite positive Borel measure on a certain space. The Hamburger moment problem \cite{hamburger1920erweiterung} specifically addresses the case where the underlying domain is the real line, i.e. it seeks $\rho\in\mathcal{M}_+(\mathbb{R})$ such that
\begin{equation}
m_k = \int_{\mathbb{R}} x^k \, d\rho(x), \qquad k \ge 0.
\end{equation}
A necessary and sufficient condition for existence is the positive semidefiniteness of all Hankel matrices formed from $(m_k)$,
\begin{equation}
H_n := \big(m_{i+j}\big)_{i,j=0}^n \succeq 0, \qquad n \in \mathbb{N},
\end{equation}
equivalently,
\begin{equation}
\sum_{i,j=0}^n c_i c_j\, m_{i+j} \ge 0 \quad \text{for all } (c_0,\ldots,c_n)\in\mathbb{R}^{n+1}.
\end{equation}
When existence holds, the representing measure need not be unique; the problem is \emph{determinate} if $\rho$ is uniquely determined by $(m_k)$, and \emph{indeterminate} otherwise. A classical sufficient condition for determinacy is \emph{Carleman’s condition} \cite{carleman1926fonctions}:
\begin{equation}
\sum_{k=1}^\infty m_{2k}^{-1/(2k)} = \infty \;\;\Longrightarrow\;\; \text{$\rho$ is unique.}
\end{equation}
For example, the univariate Gaussian $\mathcal{N}(\mu,\sigma^2)$ has moments
\begin{equation}
m_n = \sum_{k=0}^{\lfloor n/2 \rfloor} \binom{n}{2k} (2k - 1)!! \, \sigma^{2k} \mu^{\,n - 2k},
\end{equation}
and this sequence satisfies Carleman’s condition (proof follows Stirling’s estimate of the even moments); hence, the Gaussian is determinate in the Hamburger moment problem.

\medskip
\noindent\textbf{The multidimensional moment problem.}
Given $\{m_{\alpha}\}_{\alpha\in\mathbb{N}^d}$, with multi-indices $\alpha=(\alpha_1,\ldots,\alpha_d)$, we ask whether there exists a finite positive Borel measure $\rho$ on $\mathbb{R}^d$ such that
\begin{equation}
m_{\alpha} = \int_{\mathbb{R}^d} x^\alpha \, d\rho(x), 
\qquad x^\alpha := \prod_{i=1}^d x_i^{\alpha_i}.
\end{equation}
In contrast to the univariate case, the analogue of the Hankel matrix is a \emph{moment matrix} indexed by multi-indices, and positivity must hold for all real multivariate polynomials:
\begin{equation}
\sum_{\alpha,\beta} c_\alpha c_\beta\, m_{\alpha+\beta} \;\geq\; 0,
\end{equation}
for all finitely supported families $\{c_\alpha\}\subset\mathbb{R}$. This is tightly connected to sums of squares and real algebraic geometry. In particular, existence depends not only on positivity but also on support constraints defined by polynomial inequalities (semi-algebraic sets). Questions of uniqueness and determinacy in multiple dimensions are substantially subtler than in the univariate setting \cite{de2003determinate}. Next, we use \emph{slicing} to tame this problem.

\subsection{Cramér--Wold and slicing multivariate measures}
\vspace{-.1in}
We aim to characterize $\rho\in \mathcal{M}_+(\mathbb{R}^d)$ from the moments of its one–dimensional projections (slices). For $\theta\in\mathbb{S}^{d-1}$, define the pushforward $\rho_\theta := (\langle\cdot,\theta\rangle)_{\#}\rho \in \mathcal{M}_+(\mathbb{R})$ and set
\begin{equation}
m_k^\theta := \int_{\mathbb{R}} t^k \, d\rho_\theta(t)
            = \int_{\mathbb{R}^d} \langle x,\theta\rangle^k \, d\rho(x), \qquad k\in\mathbb{N}.
\end{equation}
The Cramér--Wold theorem \cite{cramer1936some} (extended via normalization to finite measures) implies: for finite positive Borel measures $\rho,\eta$ on $\mathbb{R}^d$, we have $\rho=\eta$ if and only if $\rho_\theta=\eta_\theta$ for all $\theta\in\mathbb{S}^{d-1}$.

\begin{theorem}[Sliced moment determinacy]
\label{thm:sliced-moment-determinacy}
Let $\rho\in\mathcal{M}_+(\mathbb{R}^d)$ have finite absolute moments of all orders,
$$M_n := \int_{\mathbb{R}^d}\|x\|^n\,d\rho(x) < \infty, \quad \forall n\in\mathbb{N}.$$
Assume that for every $\theta\in\mathbb{S}^{d-1}$ the univariate Hamburger moment problem for $\rho_\theta$ is \emph{determinate} (e.g., its even moments satisfy Carleman’s condition). Then $\rho$ is uniquely determined by the family of sliced moments $\{m_k^\theta : \theta\in\mathbb{S}^{d-1},\,k\in\mathbb{N}\}$.
\end{theorem}

\begin{proof}[Proof sketch]
If $\eta\in\mathcal{M}_+(\mathbb{R}^d)$ has the same sliced moments as $\rho$, then for each fixed $\theta$ the corresponding univariate moment sequences coincide, hence determinacy yields $\rho_\theta=\eta_\theta$. By Cramér--Wold (after normalizing masses, which coincide since $m_0^\theta(\rho)=m_0^\theta(\eta)$), we conclude $\rho=\eta$. The full proof is omitted due to space limitations. 
\end{proof}

\noindent\textbf{Link to multivariate moments.}
For $\alpha\in\mathbb{N}^d$, write $m_\alpha := \int_{\mathbb{R}^d} x^\alpha\,d\rho(x)$ with $x^\alpha:=\prod_{i=1}^d x_i^{\alpha_i}$. Then for each $k\in\mathbb{N}$,
\begin{equation}
\label{eq:slice-moment-expansion}
\begin{aligned}
m_k^\theta
&= \int_{\mathbb{R}^d} \langle x,\theta\rangle^k\,d\rho(x)
 = \sum_{|\alpha|=k} \binom{k}{\alpha}\,\theta^\alpha\, m_\alpha,\\
\binom{k}{\alpha}&:=\frac{k!}{\alpha_1!\cdots\alpha_d!},\qquad
\theta^\alpha:=\prod_{i=1}^d \theta_i^{\alpha_i}.
\end{aligned}
\end{equation}
Thus $m_k^\theta$ is a homogeneous polynomial of degree $k$ in $\theta$ whose coefficients are precisely the order-$k$ moments $\{m_\alpha:|\alpha|=k\}$. Knowing $\theta\mapsto m_k^\theta$ for all $\theta\in\mathbb{S}^{d-1}$ determines these coefficients uniquely (polynomial uniqueness), so the full multivariate moment sequence is recoverable degree-by-degree from sliced moments.

Capturing \emph{all} moments for arbitrary finite positive Borel measures is daunting: in $d$ dimensions, the number of monomials of total degree $\le K$ grows combinatorially as $\binom{d+K}{K}$; the associated moment matrices (of size $\binom{d+n}{n}$ at degree $n$) become large and often ill–conditioned; and mere existence already requires positivity of the Riesz functional on squares, i.e., $L(p^2)\ge 0$ for all polynomials $p$ (equivalently, positive semidefiniteness of all moment matrices). With support constraints, one further needs Positivstellensatz certificates (e.g., Putinar/Schmüdgen under Archimedean assumptions). Moreover, uniqueness is not guaranteed (moment–indeterminate laws exist), high–order moments are extremely sensitive to tail behavior and sampling noise, and finite truncations of the moment sequence need not identify the underlying measure.

To avoid these complications, we restrict our attention to an expressive yet tractable class of distributions, i.e., Gaussian mixture models (GMMs). 

\vspace{-.1in}
\subsection{The Special Case of Gaussian Mixture Models }
\label{sec:gmm}
\vspace{-.1in}

When $\rho$ is a multivariate Gaussian mixture, we have,
\vspace{-.15in}
\begin{equation}
\label{eq:gmm}
\rho \;=\; \sum_{j=1}^K \pi_j\, \mathcal{N}(\mu_j,\Sigma_j),
\vspace{-.2in}
\end{equation}
for $\pi_j>0,\ \sum_{j=1}^K \pi_j=1,\ \mu_j\in\mathbb{R}^d,\ \Sigma_j\in\mathbb{S}_{++}^d$.
This class is particularly attractive for two main reasons:
(i) finite mixtures of Gaussians are dense in the set of probability measures (weak topology) and can approximate smooth densities arbitrarily well in $L^p$ norms and even uniformly on compacts, given enough components \cite{nguyen2020approximation}, and (ii) \emph{all} raw moments of $\rho$ are explicit functions of the parameters
$\{\pi_j,\mu_j,\Sigma_j\}_{j=1}^K$ via Isserlis’/Wick’s theorem \cite{isserlis1918formula}. The multivariate
moment generating function of $X\sim \rho$, i.e., $M_X(t)=\mathbb{E}[e^{t^\top X}]$, is
\vspace{-.1in}
\begin{equation}
\label{eq:gmm-mgf}
M_X(t)
=\sum_{j=1}^K \pi_j \exp\!\Big(t^\top \mu_j+\tfrac12\, t^\top \Sigma_j t\Big),
~~t\in\mathbb{R}^d,
\vspace{-.1in}
\end{equation}
so each raw moment $m_\alpha=\partial_t^\alpha M_X(0)$ is a finite polynomial in
$\{\pi_j,\mu_j,\Sigma_j\}$. Moreover, any finite Gaussian mixture is \emph{moment-determinate}: its moment generating function
\eqref{eq:gmm-mgf} is finite for all $t\in\mathbb{R}^d$ (entire), hence the full moment sequence,
i.e., the Taylor coefficients at $t=0$, uniquely determines the distribution and thus the parameters
up to label swapping.

Importantly, vectorizing the parameters of a $K$-component, $d$-dimensional Gaussian mixture yields a high-dimensional representation: each component contributes $d$ mean entries and $d(d+1)/2$ covariance entries, plus one mixture weight, for a total of
$
K\!\left(d+\frac{d(d+1)}{2}+1\right)
$
parameters (or $K\!\left(d+\frac{d(d+1)}{2}\right)+(K-1)$ if the simplex constraint $\sum_j \pi_j=1$ is enforced). In high-dimensional settings, this parameterization quickly becomes prohibitively expensive for learning and inference, both computationally and statistically, due to its quadratic dependence $O(d^2)$ arising from the covariance parameters.

\vspace{-.1in}
\subsection{Sliced GMMs}
\vspace{-.1in}
For every direction $\theta\in\mathbb{S}^{d-1}$, the 1D pushforward is
\begin{equation}
\label{eq:sliced-gmm}
\rho_\theta \;=\; (\langle\cdot,\theta\rangle)_{\#}\rho
\;=\;\sum_{j=1}^K \pi_j\, \mathcal{N}\!\big(\theta^\top \mu_j,\ \theta^\top \Sigma_j \theta\big)
\;\in\; \mathcal{M}_+(\mathbb{R}).
\end{equation}
Hence each slice is itself a (univariate) GMM, and its $k$-th moment is
$m_k^\theta=\sum_{j=1}^K \pi_j\, \mathbb{E}_{Z\sim \mathcal{N}(\theta^\top \mu_j,\ \theta^\top \Sigma_j \theta)}[Z^k]$.

\begin{proposition}[Determinacy of sliced GMMs]
\label{prop:sliced-gmm-determinate}
Let $\rho$ be as in \eqref{eq:gmm}. Then, for every $\theta\in\mathbb{S}^{d-1}$, the univariate
moment sequence of $\rho_\theta$ satisfies Carleman’s condition and is determinate.
Consequently, by Theorem~\ref{thm:sliced-moment-determinacy}, $\rho$ is uniquely determined by the
family of sliced moments $\{m_k^\theta:\theta\in\mathbb{S}^{d-1},\,k\in\mathbb{N}\}$.
\end{proposition}

\begin{proof}[Proof idea]
Each component in \eqref{eq:sliced-gmm} is Gaussian with variance $\theta^\top\Sigma_j\theta>0$,
hence its even moments grow like $(2n-1)!!\,(\theta^\top\Sigma_j\theta)^n$ up to lower-order mean
terms. Summing over finitely many components yields
$(m_{2n}^\theta)^{1/(2n)}\le C_1 + C_2\sqrt{n}$ uniformly in $n$, so
$\sum_n (m_{2n}^\theta)^{-1/(2n)}=\infty$ (as in the Gaussian proof), implying Carleman’s condition
for $\rho_\theta$. Determinacy then follows for each slice, and Theorem~\ref{thm:sliced-moment-determinacy}
applies.
\vspace{-.1in}
\end{proof}

\subsection{Algorithmic \& Statistical Aspects of EMPEROR} 
\vspace{-.1in}
Let $\{x_i\}_{i=1}^N$ be i.i.d.\ samples from an unknown $K$-component Gaussian mixture
$\rho=\sum_{j=1}^K \pi_j\,\mathcal{N}(\mu_j,\Sigma_j)\in\mathcal{P}(\mathbb{R}^d)$. Directly estimating $\{(\pi_j,\mu_j,\Sigma_j)\}_{j=1}^K$ in $\mathbb{R}^d$ can be fragile in high dimensions (curse of dimensionality).  
We therefore fix $L$ directions $\Theta:=\{\theta_\ell\in\mathbb{S}^{d-1}\}_{\ell=1}^L$ and work with the one–dimensional pushforwards
$\rho_{\theta_\ell}:=(\langle\cdot,\theta_\ell\rangle)_{\#}\rho$. We first emphasize that our \emph{goal is not} to reconstruct the ambient parameters; rather, we seek a finite, moment-preserving \emph{descriptor} built from the parameters of the $L$ univariate GMMs. To that end, for each $\ell$ we estimate a $K$-component univariate GMM from the projected samples
$y_i^{(\ell)}:=\theta_\ell^\top x_i$, obtaining
\vspace{-.1in}
\[
\widehat{\mathcal{P}}^{(\ell)}=\big\{(\widehat{\pi}^{(\ell)}_k,\widehat{\mu}^{(\ell)}_k,\widehat{\sigma}^{(\ell)}_k)\big\}_{k=1}^K,
\vspace{-.1in}
\]
and define the sliced descriptor
$\widehat{\mathcal{S}}(\rho;\Theta):=\big\{\widehat{\mathcal{P}}^{(\ell)}\big\}_{\ell=1}^L$.
Within each slice, we fix labels by sorting components with respect to their  $\widehat{\mu}^{(\ell)}_k$ to remove intra-slice label ambiguity.

In the population model the mixture weights are \emph{slice-invariant}: $\pi_j^{\theta}=\pi_j$ for all $\theta$ (up to a permutation of components).
Enforcing this invariance during estimation induces a \emph{coupled} likelihood across slices and complicates the GMM approximation (e.g, via the expectation maximization algorithm). Because our objective is a compact representation rather than ambient parameter recovery, we \emph{do not} couple the slices and
treat $\{\widehat{\pi}^{(\ell)}_k\}$ as slice-specific parameters. This avoids cross-slice constraints while still yielding a strong, moment-rich descriptor.

To justify per-slice $K$-component fits, note that if two components $j\neq j'$ satisfy
$
\theta^\top\mu_j=\theta^\top\mu_{j'} \quad\text{and}\quad
\theta^\top\Sigma_j\theta=\theta^\top\Sigma_{j'}\theta
$,
then they \emph{collide} in the slice $\theta$. For fixed $(\mu_j,\Sigma_j)\neq(\mu_{j'},\Sigma_{j'})$, the set of $\theta\in\mathbb{S}^{d-1}$ solving these two polynomial equations is a proper (lower-dimensional) algebraic subset; hence it has surface measure zero. Consequently, for almost every $\theta$ the projected mixture has $K$ distinct components and is identifiable in 1-D. In practice, drawing several (e.g., random) directions makes collisions vanishingly unlikely.

\noindent\textbf{}{Statistical remark.} Under standard regularity assumptions for finite mixtures (identifiability and well-specified $K$), the per-slice MLEs are consistent:
$$\widehat{\mathcal{P}}^{(\ell)}\xrightarrow{p}\mathcal{P}^{(\ell)}=\{(\pi_{\sigma_\ell(k)},\theta_\ell^\top\mu_{\sigma_\ell(k)},\sqrt{\theta_\ell^\top\Sigma_{\sigma_\ell(k)}\theta_\ell})\}_{k=1}^K$$
for some permutation $\sigma_\ell$. Thus the descriptor
$\widehat{\mathcal{S}}(\rho;\Theta)$ converges (up to per-slice label swaps) to the collection of true sliced parameters, providing a finite, robust summary of $\rho$ via one–dimensional GMMs.

\noindent\textbf{Statistical error bound of EMPEROR.}
Fix a degree $k\in\mathbb{N}$ and let $M_k:=\binom{d+k-1}{k}$ be the number of monomials of total
degree $k$. Stack the multivariate moments into $m^{(k)}\in\mathbb{R}^{M_k}$, ordered by
multi–indices $\{\alpha:\ |\alpha|=k\}$, and define for directions
$\Theta=\{\theta_\ell\}_{\ell=1}^L\subset\mathbb{S}^{d-1}$ the design matrix $\Phi_k\in\mathbb{R}^{L\times M_k}$ by
\vspace{-.15in}
\[
(\Phi_k)_{\ell,\alpha}\;:=\;\binom{k}{\alpha}\,\theta_\ell^\alpha,
\qquad
y^{(k)}_\ell\;:=\;m_k^{\theta_\ell}=\sum_{|\alpha|=k}(\Phi_k)_{\ell,\alpha}\,m_\alpha,
\vspace{-.1in}
\]
so that $y^{(k)}=\Phi_k m^{(k)}$. In practice we observe
$\widehat{y}^{(k)}=y^{(k)}+\varepsilon^{(k)}$, where the errors
$\varepsilon^{(k)}=(\varepsilon^{(k)}_1,\dots,\varepsilon^{(k)}_L)^\top$ model per–slice estimation
noise (e.g., from fitting univariate GMMs). Assume
$
\mathbb{E}\big[\varepsilon^{(k)}\big]=0,\qquad
\mathrm{Cov}\!\big(\varepsilon^{(k)}\big)=\frac{\tau_k^2}{N}\,I_L
$
for some proxy variance $\tau_k^2$ (depending on $k$ and the underlying distribution) and sample
size $N$. The least–squares estimator
\[
\widehat{m}^{(k)}:=\arg\min_{u\in\mathbb{R}^{M_k}}\|\Phi_k u-\widehat{y}^{(k)}\|_2^2
\;=\;(\Phi_k^\top\Phi_k)^{-1}\Phi_k^\top \widehat{y}^{(k)}.
\]
For $\theta_\ell\overset{\text{i.i.d.}}{\sim}\mathrm{Unif}(\mathbb{S}^{d-1})$ and $L\ge M_k$,
$\mathrm{rank}(\Phi_k)=M_k$ holds almost surely, so the LS estimator is well-defined, satisfying
\begin{equation}\label{eq:stat-bound-exact}
\mathbb{E}\,\big\|\widehat{m}^{(k)}-m^{(k)}\big\|_2^2
=\frac{\tau_k^2}{N}\,\mathrm{Tr}\!\big((\Phi_k^\top\Phi_k)^{-1}\big)
\;\le\;\frac{\tau_k^2}{N}\,\frac{M_k}{\sigma_{\min}(\Phi_k)^2},
\end{equation}
where $\sigma_{\min}(\Phi_k)$ is the smallest singular value of $\Phi_k$. Moreover, by the law of large numbers for random features we have, 
$
\frac{1}{L}\,\Phi_k^\top\Phi_k \;\xrightarrow[L\to\infty]{\ \ \ \ }\;
\Sigma_k:=\mathbb{E}\!\big[\varphi_k(\theta)\varphi_k(\theta)^\top\big]$,
where $\varphi_k(\theta):=\big(\binom{k}{\alpha}\theta^\alpha\big)_{|\alpha|=k}$ and $\Sigma_k$ is positive definite (hence $\lambda_{\min}(\Sigma_k)>0$) for every fixed $(d,k)$.
With high probability for large $L$,
$\sigma_{\min}(\Phi_k)\gtrsim \sqrt{L}\,\sqrt{\lambda_{\min}(\Sigma_k)}$, so \eqref{eq:stat-bound-exact}
yields,
\vspace{-.1in}
\begin{equation}\label{eq:stat-rate}
\mathbb{E}\,\big\|\widehat{m}^{(k)}-m^{(k)}\big\|_2
\;\lesssim\; \frac{\tau_k}{\sqrt{N}}\,
\sqrt{\frac{M_k}{L}}\;\frac{1}{\sqrt{\lambda_{\min}(\Sigma_k)}},
\vspace{-.1in}
\end{equation}
i.e., for fixed degree $k$ and sample size $N$, the (root-mean-square) error decays as
$L^{-1/2}$ with the number of slices. Summing \eqref{eq:stat-rate} over $k\le K$ gives the same
$L^{-1/2}$ and $N^{-1/2}$ scaling (up to $\sum_{k=0}^{K} M_k=\binom{d+K}{K}$). In practice, ridge
regularization replaces $(\Phi_k^\top\Phi_k)^{-1}$ by $(\Phi_k^\top\Phi_k+\lambda I)^{-1}$ producing a bias-variance trade-off with the decay rate of $L^{-1/2}$.

\vspace{-.1in}
\section{Experiments}
\label{sec:exp}
\vspace{-.1in}

To evaluate EMPEROR, we conduct two tasks. (1) Point cloud classification on Point Cloud MNIST \cite{pointcloudmnist2d2025} and ModelNet40 \cite{wu20153d}, using distribution descriptors in both the ambient space (2D and 3D) and the embedding space of a pretrained point cloud Transformer \cite{guo2021pct} (256 dimensions). (2) Image representation analysis with a pretrained Vision Transformer \cite{dosovitskiyimage} (ImageNet), where we classify samples from the ClipArt and Painting domains of DomainNet \cite{peng2019moment} using descriptors extracted at different ViT layers (3-12).

We perform classification tasks using the representations/descriptors extracted by EMPEROR along with other baselines, including: Global Average Pool (GAP) \cite{lin2013network}, Generalized Max-Pooling (GMP) \cite{murray2014generalized}, Generalized Mean Pooling (GeM) \cite{radenovic2018fine}, Covariance Pooling \cite{acharya2018covariance}, Featurewise Sort Pooling (FSPool) \cite{Zhang2020FSPool}, Wasserstein Embedding (WE) \cite{kolouri2021wasserstein}. Importantly, the latter two baselines are designed to capture higher moments of the features' distribution. 

\vspace{-.1in}
\subsection{Point Cloud Classification}
\vspace{-.1in}
 We evaluate point cloud classification using EMPEROR and baseline descriptors. Table \ref{tab:pc} reports results averaged over three runs, showing that EMPEROR yields strong performance even without a backbone (ambient space).
\begin{table}[H]
\vspace{-.1in}
\begin{tabular}{c|cc|cc}
method&\multicolumn{2}{c|}{Point Cloud Mnist 2D}& \multicolumn{2}{c}{ModelNet40}\\
\hline
    & \scriptsize Identity & \scriptsize PCT &\scriptsize Identity & \scriptsize PCT \\
    \cline{2-5}
    GAP& 0.2581 &0.9700 &0.0405 &0.7216\\
    GMP &  0.4244 &0.6373 &0.3254 &0.7277\\
    GeM&0.2881&0.9092&0.2597&0.7561 \\
    Cov &0.4180 &0.9710&0.2549&0.8071 \\
    FSpool &0.2788&0.9252 &0.3051&0.7387 \\
    WE &0.9478 &0.9712 &0.8448 &0.8489 \\
    \hdashline
    \rule{0pt}{1.05\normalbaselineskip}
    EMPEROR&\textbf{0.9643} &\textbf{0.9717}&\textbf{0.8517}&\textbf{0.8674} \\
\hline
\end{tabular}
\vspace{-.1in}
\caption{Results of different distribution descriptors on PC MNIST (2D), and ModelNet40 (3D) datasets, with identity backbone as well as a PC Transformer (PCT) backbone.}
\label{tab:pc}
\vspace{-0.2in}
\end{table}
\vspace{-0.1in}
\subsection{Image Classification}
\vspace{-0.1in}
For image experiments, we pass Clipart and Painting images from DomainNet through a pretrained ViT, extract token representations at each layer, and apply the distribution descriptors before a linear classifier. The results are shown in Figure \ref{fig:fig1}. We see that EMPEROR achieves competitive performance (even with the CLS token) and yields more robust representations across all layers.
\vspace{-0.17in}
\begin{figure}[H]
    \centering
    \includegraphics[width=\linewidth,trim={0 0 0 150},clip]{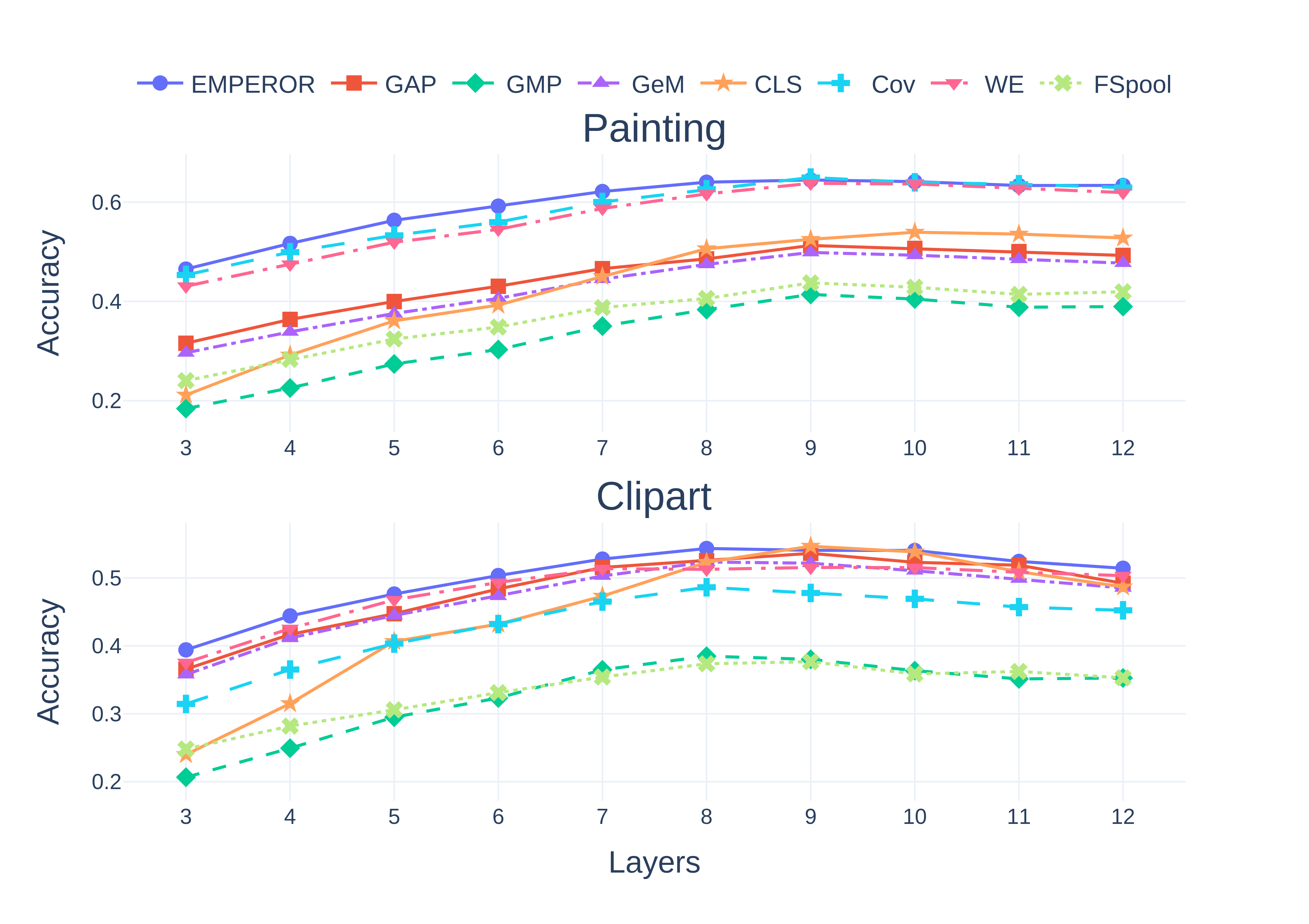}
    \vspace{-0.4in}
    \caption{Image classification results across different layers on the Painting (top) and Clipart (bottom) datasets.}
    \label{fig:fig1}
\end{figure}
\newpage
\clearpage
\bibliographystyle{IEEEbib}
\bibliography{strings,refs}

\begin{thebibliography}{10}

\bibitem{lin2013network}
Min Lin, Qiang Chen, and Shuicheng Yan,
\newblock ``Network in network,''
\newblock {\em arXiv preprint arXiv:1312.4400}, 2013.

\bibitem{dosovitskiyimage}
Alexey Dosovitskiy, Lucas Beyer, Alexander Kolesnikov, Dirk Weissenborn,
  Xiaohua Zhai, Thomas Unterthiner, Mostafa Dehghani, Matthias Minderer, Georg
  Heigold, Sylvain Gelly, et~al.,
\newblock ``An image is worth 16x16 words: Transformers for image recognition
  at scale,''
\newblock in {\em International Conference on Learning Representations}.

\bibitem{Zhang2020FSPool}
Yan Zhang, Jonathon Hare, and Adam Prügel-Bennett,
\newblock ``Fspool: Learning set representations with featurewise sort
  pooling,''
\newblock in {\em International Conference on Learning Representations}, 2020.

\bibitem{mialon2021a}
Gr{\'e}goire Mialon, Dexiong Chen, Alexandre d'Aspremont, and Julien Mairal,
\newblock ``A trainable optimal transport embedding for feature aggregation and
  its relationship to attention,''
\newblock in {\em International Conference on Learning Representations}, 2021.

\bibitem{kolouri2021wasserstein}
Soheil Kolouri, Navid NaderiAlizadeh, Gustavo~K Rohde, and Heiko Hoffmann,
\newblock ``Wasserstein embedding for graph learning,''
\newblock in {\em International Conference on Learning Representations}, 2021.

\bibitem{naderializadeh2021PSWE}
Navid NaderiAlizadeh, Joseph~F. Comer, Reed~W Andrews, Heiko Hoffmann, and
  Soheil Kolouri,
\newblock ``Pooling by sliced-{Wasserstein} embedding,''
\newblock in {\em Thirty-Fifth Conference on Neural Information Processing
  Systems}, 2021.

\bibitem{kothapalli2024equivariant}
Abihith Kothapalli, Ashkan Shahbazi, Xinran Liu, Robert Sheng, and Soheil
  Kolouri,
\newblock ``Equivariant vs. invariant layers: A comparison of backbone and
  pooling for point cloud classification,''
\newblock in {\em ICML 2024 Workshop on Geometry-grounded Representation
  Learning and Generative Modeling}, 2024.

\bibitem{stieltjes1894recherches}
T-J Stieltjes,
\newblock ``Recherches sur les fractions continues,''
\newblock in {\em Annales de la Facult{\'e} des sciences de Toulouse:
  Math{\'e}matiques}, 1894, vol.~8, pp. J1--J122.

\bibitem{shohat1950problem}
James Shohat, James~Alexander Shohat, and Jacob~David Tamarkin,
\newblock {\em The problem of moments}, vol.~1,
\newblock American Mathematical Society (RI), 1950.

\bibitem{aheizer1965classical}
Naum~Ilji{\v{c}} Aheizer and N~Kemmer,
\newblock {\em The classical moment problem and some related questions in
  analysis},
\newblock Oliver \& Boyd Edinburgh, 1965.

\bibitem{hamburger1920erweiterung}
Hans Hamburger,
\newblock ``{\"U}ber eine erweiterung des stieltjesschen momentenproblems,''
\newblock {\em Mathematische Annalen}, vol. 81, no. 2, pp. 235--319, 1920.

\bibitem{carleman1926fonctions}
Torsten Carleman,
\newblock {\em Les Fonctions quasi analytiques: le{\c{c}}ons profess{\'e}es au
  College de France},
\newblock Gauthier-Villars, 1926.

\bibitem{de2003determinate}
Marcel de~Jeu,
\newblock ``Determinate multidimensional measures, the extended carleman
  theorem and quasi-analytic weights,''
\newblock {\em The annals of probability}, vol. 31, no. 3, pp. 1205--1227,
  2003.

\bibitem{cramer1936some}
Harald Cram{\'e}r and Herman Wold,
\newblock ``Some theorems on distribution functions,''
\newblock {\em Journal of the London Mathematical Society}, vol. 1, no. 4, pp.
  290--294, 1936.

\bibitem{nguyen2020approximation}
T~Tin Nguyen, Hien~D Nguyen, Faicel Chamroukhi, and Geoffrey~J McLachlan,
\newblock ``Approximation by finite mixtures of continuous density functions
  that vanish at infinity,''
\newblock {\em Cogent Mathematics \& Statistics}, vol. 7, no. 1, pp. 1750861,
  2020.

\bibitem{isserlis1918formula}
Leon Isserlis,
\newblock ``On a formula for the product-moment coefficient of any order of a
  normal frequency distribution in any number of variables,''
\newblock {\em Biometrika}, vol. 12, no. 1/2, pp. 134--139, 1918.

\bibitem{pointcloudmnist2d2025}
Cristian Garcia,
\newblock ``Point cloud mnist 2d,'' Kaggle dataset, 2025,
\newblock Based on MNIST; non-zero pixels converted into 2D point clouds.

\bibitem{wu20153d}
Zhirong Wu, Shuran Song, Aditya Khosla, Fisher Yu, Linguang Zhang, Xiaoou Tang,
  and Jianxiong Xiao,
\newblock ``3d shapenets: A deep representation for volumetric shapes,''
\newblock in {\em Proceedings of the IEEE conference on computer vision and
  pattern recognition}, 2015.

\bibitem{guo2021pct}
Meng-Hao Guo, Jun-Xiong Cai, Zheng-Ning Liu, Tai-Jiang Mu, Ralph~R Martin, and
  Shi-Min Hu,
\newblock ``Pct: Point cloud transformer,''
\newblock {\em Computational visual media}, vol. 7, no. 2, pp. 187--199, 2021.

\bibitem{peng2019moment}
Xingchao Peng, Qinxun Bai, Xide Xia, Zijun Huang, Kate Saenko, and Bo~Wang,
\newblock ``Moment matching for multi-source domain adaptation,''
\newblock in {\em Proceedings of the IEEE/CVF international conference on
  computer vision}, 2019, pp. 1406--1415.

\bibitem{murray2014generalized}
Naila Murray and Florent Perronnin,
\newblock ``Generalized max pooling,''
\newblock in {\em Proceedings of the IEEE conference on computer vision and
  pattern recognition}, 2014.

\bibitem{radenovic2018fine}
Filip Radenovi{\'c}, Giorgos Tolias, and Ond{\v{r}}ej Chum,
\newblock ``Fine-tuning cnn image retrieval with no human annotation,''
\newblock {\em IEEE transactions on pattern analysis and machine intelligence},
  vol. 41, no. 7, pp. 1655--1668, 2018.

\bibitem{acharya2018covariance}
Dinesh Acharya, Zhiwu Huang, Danda Pani~Paudel, and Luc Van~Gool,
\newblock ``Covariance pooling for facial expression recognition,''
\newblock in {\em Proceedings of the IEEE conference on computer vision and
  pattern recognition workshops}, 2018, pp. 367--374.

\end{thebibliography}

\end{document}